\documentclass[10pt]{article} 

\usepackage[preprint]{rlj} 

\usepackage{amssymb,amsthm}
\usepackage{mathtools}
\usepackage{mathrsfs}
\usepackage{graphicx}
\usepackage{subcaption}
\usepackage[space]{grffile}
\usepackage{algorithm,algorithmic}
\usepackage{tikz}
\usepackage{booktabs}
\usepackage{multirow}
\usepackage{enumitem}
\usetikzlibrary{positioning,arrows.meta,fit,calc,decorations.pathreplacing}

\definecolor{cSignal}{HTML}{2166AC}   
\definecolor{cGrad}{HTML}{8856A7}     
\definecolor{cExpert}{HTML}{117733}   
\definecolor{cWaste}{HTML}{B2182B}    
\definecolor{cConsensus}{HTML}{088F8F} 

\definecolor{cIncumbent}{HTML}{7B68EE}   
\definecolor{cBatch}{HTML}{E8871E}       
\definecolor{cExpert}{HTML}{117733}      
\definecolor{cOut}{HTML}{D6604D}         

\newtheorem{theorem}{Theorem}

\newtheorem{proposition}[theorem]{Proposition}
\newtheorem{corollary}[theorem]{Corollary}
\newtheorem{definition}{Definition}
\newtheorem{remark}{Remark}

\title{Optimize Wider, Not Deeper:\\ Consensus Aggregation for Policy Optimization}

\setrunningtitle{Optimize Wider, Not Deeper}

\author{Zelal Su (Lain) Mustafaoglu$^{1}$, Sungyoung Lee$^{2}$, Eshan Balachandar$^{1}$, Risto Miikkulainen$^{1}$, Keshav Pingali$^{1}$}

\emails{zsm@utexas.edu, sylee@utexas.edu, eshan@cs.utexas.edu, risto@cs.utexas.edu, pingali@cs.utexas.edu}{
}

\affiliations{
$^{1}$\textbf{Department of Computer Science, University of Texas at Austin}
$^{2}$\textbf{Department of Electrical and Computer Engineering, University of Texas at Austin}
}

\contribution{
    A Fisher-geometric decomposition of the PPO update into signal (natural gradient projection) and waste (Fisher-orthogonal residual). Signal saturates within a few epochs while waste accumulates, formalizing the optimization-depth dilemma.
    }
    {
    The decomposition explains why increasing epoch count yields diminishing returns and why averaging PPO copies improves KL efficiency: shared signal is preserved while path-dependent waste partially cancels.
    }

\contribution{
    CAPO (Consensus Aggregation for Policy Optimization) algorithm: runs the PPO inner loop $K$ times independently on the same batch and aggregates them into a consensus, either in parameter space or in the natural parameter space of the policy distribution.
    }
    {
    In natural parameter space, the consensus provably achieves higher KL-penalized surrogate and tighter trust region compliance than the mean expert (Theorem~\ref{thm:consensus}). The natural-parameter consensus is exactly the LogOP, which outperforms Euclidean averaging when expert variances diverge.
    }

\contribution{
    Empirical validation across continuous control benchmarks in Gymnasium, providing empirical insights into signal--waste decomposition and demonstrating improvements over PPO and compute-matched baselines.
    }
    {
    CAPO outperforms PPO on five out of six tasks by up to $8.6\!\times$, while PPO with $K\!\times$ more epochs degrades on every task due to waste accumulation. Parameter averaging reduces waste by $2$--$17\%$, across all tasks; aggregation in natural parameter space (LogOP) achieves $46\%$ waste reduction on Humanoid, where precision-weighting across more parameters amplifies the benefit due to the high dimensionality of the task.
    }

\keywords{RL, On-Policy RL, Policy Optimization, Trust Region Methods, Proximal Policy Optimization (PPO), Natural Gradients, Information Geometry, Continuous Control}

\summary{Proximal policy optimization (PPO) approximates the trust region update using multiple epochs of clipped SGD. Each epoch may drift further from the natural gradient direction, creating path-dependent noise. To understand this drift, we can use Fisher information geometry to decompose policy updates into \emph{signal} (the natural gradient projection) and \emph{waste} (the Fisher-orthogonal residual that consumes trust region budget without first-order surrogate improvement). Empirically, signal saturates but waste grows with additional epochs, creating an \emph{optimization-depth dilemma}. We propose \textbf{Consensus Aggregation for Policy Optimization (CAPO)}, which redirects compute from depth to width: $K$ PPO replicates are optimized on the same batch, differing only in minibatch shuffling order, and then aggregated into a consensus. We study aggregation in two spaces: Euclidean parameter space, and the natural parameter space of the policy distribution via the logarithmic opinion pool. In natural parameter space, the consensus provably achieves higher KL-penalized surrogate and tighter trust region compliance than the mean expert; parameter averaging inherits these guarantees approximately. On continuous control tasks, CAPO outperforms PPO and compute-matched deeper baselines under fixed sample budgets by up to $8.6\!\times$. CAPO demonstrates that policy optimization can be improved by optimizing wider, rather than deeper, without additional environment interactions.
}

\begin{document}

\makeCover  
\maketitle  

\begin{abstract}
Proximal policy optimization (PPO) approximates the trust region update using multiple epochs of clipped SGD. Each epoch may drift further from the natural gradient direction, creating path-dependent noise. To understand this drift, we can use Fisher information geometry to decompose policy updates into \emph{signal} (the natural gradient projection) and \emph{waste} (the Fisher-orthogonal residual that consumes trust region budget without first-order surrogate improvement). Empirically, signal saturates but waste grows with additional epochs, creating an \emph{optimization-depth dilemma}. We propose \textbf{Consensus Aggregation for Policy Optimization (CAPO)}, which redirects compute from depth to width: $K$ PPO replicates are optimized on the same batch, differing only in minibatch shuffling order, and then aggregated into a consensus. We study aggregation in two spaces: Euclidean parameter space, and the natural parameter space of the policy distribution via the logarithmic opinion pool. In natural parameter space, the consensus provably achieves higher KL-penalized surrogate and tighter trust region compliance than the mean expert; parameter averaging inherits these guarantees approximately. On continuous control tasks, CAPO outperforms PPO and compute-matched deeper baselines under fixed sample budgets by up to $8.6\!\times$. CAPO demonstrates that policy optimization can be improved by optimizing wider, rather than deeper, without additional environment interactions.
\end{abstract}

\section{Introduction}
\label{sec:intro}

\begin{figure}[t]
\centering
\begin{subfigure}[b]{0.48\linewidth}
\centering
\resizebox{\linewidth}{!}{%
\begin{tikzpicture}[
  vec/.style   = {-{Latex[length=2.6mm]}, line width=0.9pt},
  thickvec/.style = {-{Latex[length=3mm]}, line width=1.2pt},
  lab/.style   = {font=\small},
  slab/.style  = {font=\scriptsize},
]

\coordinate (O) at (0, 0);

\coordinate (E1) at (-3.2, 3.5);   
\coordinate (E2) at ( 2.8, 4.8);   

\coordinate (S1) at (0, 3.5);
\coordinate (S2) at (0, 4.8);

\coordinate (AVG) at (-0.35, 4.15);

\draw[thick, black!50] (0, 2.5) ellipse (4.5cm and 3.8cm);
\node[black!60, fill=white, inner sep=2pt, anchor=center, font=\normalsize] at (-3.2, 6.2) {%
  $\mathrm{KL}(\pi_t \| \pi) \leq \delta$};

\draw[thickvec, cSignal] (O) -- (0, 5.7);
\node[lab, cSignal, anchor=south east] at (-0.1, 5.7) {$F^{-1}g$};

\draw[thickvec, cGrad] (O) -- (3.5, 1.9);
\node[lab, cGrad, anchor=west] at (3.55, 1.9) {$g$};

\draw[vec, cExpert, opacity=0.35] (O) -- (E1);
\node[slab, cExpert!85!black, anchor=south east]
  at ($(O)!0.38!(E1) + (-0.05, 0.05)$) {$\Delta^{1}$};

\draw[vec, cWaste, densely dashed, line width=1pt] (S1) -- (E1);
\node[slab, cWaste, anchor=south]
  at ($(S1)!0.50!(E1) + (0, 0.12)$) {waste$^{1}$};

\draw[black!45, line width=0.5pt]
  ($(S1)+(-0.3, 0)$) -- ++(0, -0.3) -- ++(0.3, 0);

\fill[cSignal] (S1) circle (2.8pt);
\node[slab, cSignal!80!black, anchor=west] at ($(S1)+(0.18, 0)$) {signal$^{1}$};

\fill[cExpert] (E1) circle (3pt);
\node[cExpert!80!black, font=\small\bfseries, anchor=south east]
  at ($(E1)+(-0.12, 0.12)$) {$\theta^{1}$};

\draw[vec, cExpert, opacity=0.35] (O) -- (E2);
\node[slab, cExpert!85!black, anchor=south west]
  at ($(O)!0.38!(E2) + (0.05, 0.05)$) {$\Delta^{2}$};

\draw[vec, cWaste, densely dashed, line width=1pt] (S2) -- (E2);
\node[slab, cWaste, anchor=south]
  at ($(S2)!0.50!(E2) + (0, 0.12)$) {waste$^{2}$};

\draw[black!45, line width=0.5pt]
  ($(S2)+(0.3, 0)$) -- ++(0, -0.3) -- ++(-0.3, 0);

\fill[cSignal] (S2) circle (2.8pt);
\node[slab, cSignal!80!black, anchor=east] at ($(S2)+(-0.18, 0)$) {signal$^{2}$};

\fill[cExpert] (E2) circle (3pt);
\node[cExpert!80!black, font=\small\bfseries, anchor=west]
  at ($(E2)+(0.18, 0)$) {$\theta^{2}$};

\draw[thickvec, cConsensus, line width=2pt] (O) -- (AVG);
\fill[white] (AVG) circle (5pt);              
\fill[cConsensus!80!black] (AVG) circle (3.5pt);
\node[cConsensus!75!black, font=\small\bfseries, anchor=east]
  at ($(AVG)+(-0.18, 0)$) {$\theta^{\mathrm{agg}}$};

\fill[black] (O) circle (2.2pt);
\node[lab, anchor=north] at ($(O)+(0, -0.15)$) {$\theta_t$};

\end{tikzpicture}%
}
\caption{\textbf{Signal--waste decomposition in the trust region.}}
\label{fig:signal-waste}
\end{subfigure}
\hfill
\begin{subfigure}[b]{0.48\linewidth}
\centering
\includegraphics[width=\linewidth]{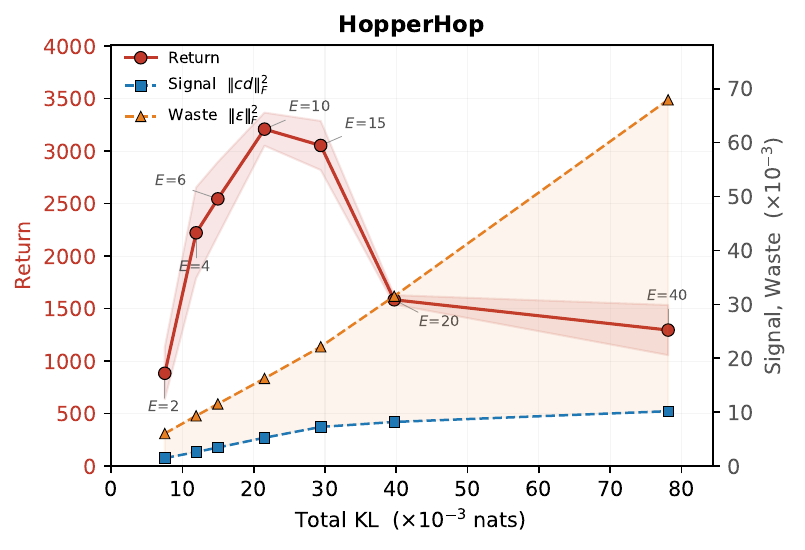}
\caption{\textbf{Return, signal, waste vs.\ total KL of the final PPO update.} (Hopper, 3 seeds). Return (red, left axis) peaks at $E\!=\!10$ then collapses; signal KL (blue) saturates while waste KL (orange) grows $11\!\times$.}
\label{fig:surr-kl-pareto}
\end{subfigure}%
\caption{\textbf{Geometry of CAPO updates.}
\textbf{(a)}~ REINFORCE moves along the gradient $g$ ({\color{cGrad}solid}) (approximately), while TRPO moves along the natural gradient $F^{-1}g$ (approximately). Each PPO expert's move $\Delta^k$ decomposes into a \emph{signal} ({\color{cSignal}$\bullet$}\,, projection onto $F^{-1}g$) and
\emph{waste} ({\color{cWaste}dashed}, Fisher-orthogonal residual).
In this diagram, waste points left for expert $\theta^1$ and right for expert $\theta^2$ so averaging reduces waste, and the consensus $\theta^{agg}$ lies closer to $F^{-1}g$ with lower KL than either expert
(Theorem~\ref{thm:consensus}).
\textbf{(b)}~Return and signal--waste KL decomposition of the final PPO update for varying epoch counts $E$ on Hopper, mean over 3 seeds.
Return peaks at $E\!=\!10$ then collapses as waste overwhelms the trust region budget.}
\label{fig:geometry}
\end{figure}
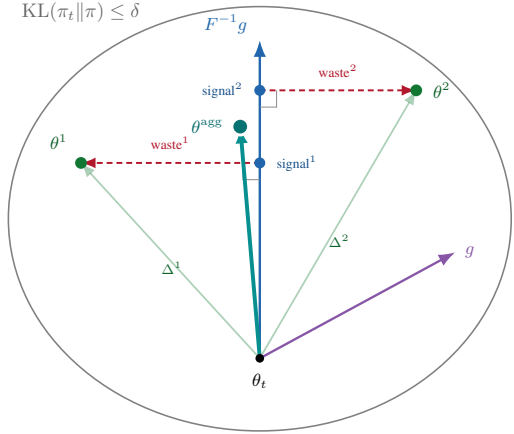
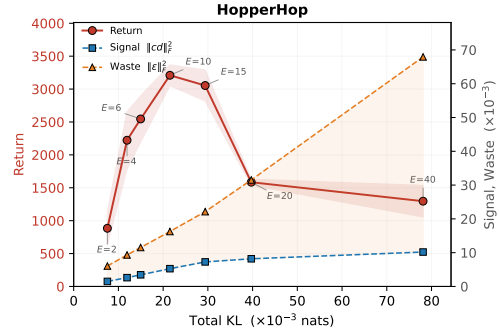

Policy gradient methods are used in reinforcement learning to train policies for tasks in problem domains ranging from robotics~\citep{rudin2022learning} to large language model alignment~\citep{ouyang2022training}. A neural network called the {\em policy network} is used to map each state $s$ to a distribution $\pi(\cdot |s)$ over actions, and the distribution is sampled to determine the next action for the agent. The objective of training is to find values for policy network parameters $\theta^{*}$ that maximize the expected total reward. 

Like most complex nonlinear optimization problems, this problem has to be solved using heuristic search. The initial point $\theta_0$ in the parameter space is usually chosen heuristically, and some rule is used iteratively to move from the current search point $\theta_t$ to the next search point $\theta_{t{+}1}$ (by convention, a subscript always refer to a timestep but we omit the subscript when it is obvious from the context). It is useful to think of this move as choosing a {\em direction} followed by a {\em step} in that direction. The move can be guided by the value of the function (and gradient, if available) at the previous and current time-steps. The notation $\pi_t$ refers to the policy induced by the policy network parameters $\theta_t$ at step $t$.

The classic REINFORCE method~\citep{sutton1998introduction} uses a Monte Carlo estimate of the gradient at $\theta_t$, and moves in the direction of the gradient with a step size determined by a heuristic learning rate. In Figure~\ref{fig:signal-waste}, this direction is shown conceptually by the line labeled $g$. Although easy to implement, the gradient estimates have high variance so moves made by REINFORCE are noisy. 

This problem led to the development of {\em trust region methods} like TRPO~\citep{schulman2015trpo}, which introduces explicit proximity constraints on updates to determine step sizes in a more principled way. TRPO requires the expectation of the KL-divergence $D_{KL}(\pi_{t}(s),\pi_{t{+}1}(s))$ over all states $s$ 
(written as $\overline{KL}(\pi_{t},\pi_{t{+}1})$), to be bounded, making policy optimization a {\em constrained} optimization problem. Following \cite{amari2000methods},
\cite{kakade2001naturalpg}  showed that the optimal direction for this problem is the {\em natural gradient} $F^{-1}g$ where $F$ is the Fisher information matrix and $g$ is the gradient. Intuitively, the natural gradient is the direction of steepest improvement per unit of KL-divergence cost, and is shown by the line labeled $F^{-1}g$ in Figure~\ref{fig:geometry}(a). TRPO computes this direction approximately using the conjugate gradient method followed by a backtracking line search, which is computationally expensive. In contrast, PPO~\citep{schulman2017ppo}  trades trust region optimality for efficiency and reduced complexity by approximating the update with $E$ epochs of clipped minibatch SGD, with $M$ minibatches per epoch. This simplification comes at a cost: each additional epoch may drift further from the natural gradient direction. As a result, PPO is a {\em noisy estimator} of the KL-constrained natural gradient step. 

Using Fisher information geometry, we can decompose any update into a \emph{signal} component, which is the projection onto the natural gradient, and an orthogonal \emph{waste} component, as shown in Figure~\ref{fig:signal-waste} for two updates $\theta^1$ and $\theta^2$ (projection is made more precise in Section~\ref{sec:depth-dilemma} using Fisher inner product). Correspondingly, the KL cost of the update splits into signal KL and waste KL (Proposition~\ref{prop:separation}).  
Figure~\ref{fig:surr-kl-pareto} shows the KL cost decomposition for the \textit{final} PPO update\footnote{To be precise, this is the update from $\theta_{f{-}1}$ to $\theta_f$ where $\theta_f$ is the final search point.} for Hopper, a standard control benchmark. Each point is PPO using a different epoch count $E$; the $x$-axis is the total KL cost. Return (red, left axis) peaks at $E\!=\!10$ and collapses beyond it. The right axis decomposes total KL cost into signal KL, (blue, $\|cd\|_F^2$) and waste KL (orange, $\|\epsilon\|_F^2$). Signal KL saturates after a few epochs, but waste KL grows $11\!\times$ from $E\!=\!2$ to $E\!=\!40$, consuming the trust region budget without contributing to the surrogate gain. 

The classical remedy for reducing noise is to average estimates. \textbf{CAPO} (Consensus Aggregation for Policy Optimization) does exactly this: given a fixed batch, it runs $K$ copies of the PPO optimizer, with the same data and incumbent but different minibatch orderings, and aggregates them into a single consensus policy. The aggregation can be performed in parameter space or in the natural parameter space of the policy distribution via the logarithmic opinion pool. In both cases, waste partially cancels while signal is preserved, as shown by the point labeled $\theta^{agg}$ in Figure~\ref{fig:signal-waste}. 

Our contributions are:
\begin{enumerate}
    \item \textbf{Signal--waste decomposition} (Section~\ref{sec:depth-dilemma}): A Fisher-geometric decomposition of the noisy  PPO update that formalizes the optimization-depth dilemma and explains why compute-matched optimization (PPO-$K\!\times$) degrades.
    \item \textbf{CAPO: a consensus operator for trust region policy optimization}   (Section~\ref{sec:method}). $K$ PPO copies on the same batch, aggregated into a consensus that provably improves on the mean individual expert in both surrogate value and trust region compliance (Theorem~\ref{thm:consensus}). 
    \item \textbf{Empirical validation} (Section~\ref{sec:experiments}): 
    \begin{itemize}
    \item CAPO (LogOP) outperforms PPO and compute-matched baselines on five out of six Gymnasium continuous control benchmarks by up to $+71\%$ on low-dimensional tasks and $8.6\!\times$ on Humanoid, a high-dimensional task. 
    \item Parameter averaging reduces waste by $2$--$17\%$, across all tasks; aggregation in natural parameter space (LogOP) achieves $46\%$ waste reduction on Humanoid, where precision-weighting across more parameters amplifies the benefit due to the high dimensionality of the task.
    \item The only overhead is $K\!\times$ gradient computation, which is embarrassingly parallel and requires no additional environment interaction. As a result, end-to-end time increases only by about 25\% on the average even for $K=4$ (Section~\ref{app:timing}). 
    \end{itemize}
\end{enumerate}

\section{Preliminaries}
\label{sec:prelims}

\label{sec:trpo}
This section describes TRPO and PPO formally, and motivates CAPO. 

\subsection{Trust Region Policy Optimization}

The KL-constrained trust region policy optimization step can be described concisely as follows: if the current policy is $\pi_t$, find $\pi_{t{+}1}$ that maximizes the expectation shown below, subject to the bound on the expected KL-divergence between $\pi_t$ and $\pi_{t{+}1}$. 

\begin{equation}\label{eq:ideal}
    \pi_{t{+}1} = \mathsf{U}^*_\delta(\pi_t, \mathcal{B}) = \arg\max_\pi \; \hat{\mathbb{E}}_{s,a \sim \mathcal{B}}\!\left[\frac{\pi(a|s)}{\pi_t(a|s)} \hat{A}^{\pi_t}(s,a)\right] \quad \text{s.t.} \quad \overline{\mathrm{KL}}(\pi_t \| \pi) \leq \delta.
\end{equation}

Here, $\mathcal{B}$ is the batch collected by the current policy $\pi_t$,  the trust region bound is $\delta$, and $\hat{A}$ are advantages for (state, action) pairs estimated via generalized advantage estimation (GAE) \citep{schulman2016gae}. Computing $\mathsf{U}^*_\delta$ exactly is intractable, so this optimization problem is solved approximately. 


TRPO~\citep{schulman2015trpo} solves a quadratic approximation (the surrogate) to~\eqref{eq:ideal} using the Fisher information matrix $F(\theta) = \mathbb{E}_{s,a\sim\pi_\theta}[\nabla\!\log\pi_\theta(a|s)\,\nabla\!\log\pi_\theta(a|s)^\top]$, which captures the local curvature of the policy's distribution space. The KL constraint becomes $\frac{1}{2}\Delta^\top\! F\,\Delta \leq \delta$, where $\Delta$ is the update, yielding the natural gradient direction $\Delta^* \propto F^{-1}g$. TRPO computes the update direction via conjugate gradient, and then performs a line search. 

PPO~\citep{schulman2017ppo} replaces the KL constraint with a clipped surrogate:
\begin{equation}\label{eq:ppo}
    L^{\text{CLIP}}(\theta) = \hat{\mathbb{E}}_t\!\left[\min\!\left(\rho_t(\theta) \hat{A}_t,\; \text{clip}(\rho_t(\theta), 1\!-\!\epsilon, 1\!+\!\epsilon) \hat{A}_t\right)\right],
\end{equation}
where $\rho_t(\theta) = \pi(a_t|s_t)/\pi_t(a_t|s_t)$ is the importance sampling ratio and $\epsilon$ is the clip threshold. PPO runs $E$ epochs of minibatch SGD ($M$ minibatches per epoch, yielding $EM$ sequential gradient steps). This is computationally cheaper than TRPO but only approximately enforces the trust region: clipping truncates large likelihood ratios but does not bound KL divergence. The gap between TRPO's geometrically principled update and PPO's approximation is what CAPO exploits.

\subsection{The Optimization-Depth Dilemma}
\label{sec:depth-dilemma}

To determine much of PPO's update is wasted, it is useful to compute the Fisher signal--waste decomposition, defined formally below.

\begin{definition}[Fisher signal--waste decomposition]\label{def:decomp}
Let $\hat{d} = F^{-1}g\,/\,\|F^{-1}g\|_F$ be the unit natural gradient direction in the Fisher norm $\|v\|_F = \sqrt{v^\top\! F\, v}$. Any update $\Delta = (\theta {-} \theta_t)$ decomposes uniquely as
\begin{equation}\label{eq:decomp}
  \Delta \;=\; \underbrace{c\, \hat{d}}_{\text{signal}}
             \;+\; \underbrace{\epsilon\vphantom{\hat{d}}}_{\text{waste}},
  \qquad
  c = \langle \Delta, \hat{d}\rangle_F, \qquad
  \langle \epsilon, \hat{d}\rangle_F = 0,
\end{equation}
where $\langle u,v\rangle_F = u^\top\! F\,v$ is the Fisher inner product and $c$ is the projection on the natural gradient. 
\end{definition}

The consensus update decomposes as $\bar{\Delta} = \bar{c}\,\hat{d} + \bar{\epsilon}$ with $\bar{c} = \frac{1}{K}\sum_k c^k$ and $\bar{\epsilon} = \frac{1}{K}\sum_k \epsilon^k$.

To keep the notation simple, we use $c$ and $\epsilon$ for projections and residuals of both the updates and the total KL since the context and type always makes it clear which one is being referred to. Signal KL and waste KL are measured in KL units ($\times 10^{-3}$ nats). We define \textit{Fisher alignment} $\alpha = c^2/(c^2 + \|\epsilon\|_F^2)$ to measure the fraction of the total KL budget spent on the natural gradient direction. \emph{Fisher cosine similarity} $\cos_F(\Delta) = \langle \Delta, \hat{d}\rangle_F / \|\Delta\|_F$ measures directional alignment.

\begin{proposition}[Signal--waste separation]\label{prop:separation}
Under decomposition~\eqref{eq:decomp}:
\begin{enumerate}[label=(\roman*),nosep]
  \item \textbf{First-order surrogate improvement depends only on signal.}\;
        $g^\top \Delta = c\,\|F^{-1}g\|_F$.\;
        The waste $\epsilon$ contributes nothing.
  \item \textbf{KL cost is additive.}\;
        $\tfrac{1}{2}\Delta^\top F\,\Delta
         = \tfrac{1}{2}c^2 + \tfrac{1}{2}\|\epsilon\|_F^2$.\;
        The signal costs $\tfrac{1}{2}c^2$ in KL;
        the waste costs $\tfrac{1}{2}\|\epsilon\|_F^2$ in KL with no first-order return.
\end{enumerate}
\end{proposition}

\begin{proof}
Using $g = \|F^{-1}g\|_F* F\hat{d}\,$ and $\Delta_k = c_k\hat{d} + \epsilon_k$:

\noindent\textbf{(i)} $g^\top\!\Delta_k = c_k\|F^{-1}g\|_F\,\underbrace{\hat{d}^\top\!F\,\hat{d}}_{1} + \|F^{-1}g\|_F\,\underbrace{\hat{d}^\top\!F\,\epsilon_k}_{0} = c_k\|F^{-1}g\|_F$.

\noindent\textbf{(ii)} By Pythagoras's theorem in the Fisher norm (\cite{amari2000methods}): $\Delta_k^\top\!F\,\Delta_k = c_k^2 + \|\epsilon_k\|_F^2$ (cross-term vanishes by orthogonality).\qedhere
\end{proof}

In theory, TRPO has $\epsilon = 0$ by construction since it follows the natural gradient $F^{-1}g$. PPO experts have $\epsilon \neq 0$. 

Table~\ref{tab:signal-saturation-body} tracks signal KL and waste KL as $E$ varies on the Gymnasium Hopper task with $(64,64)$ networks\footnote{Figure~\ref{fig:surr-kl-pareto} showed this data pictorially.}. At $E\!=\!4$, alignment is $\alpha = 0.22$: only $22\%$ of the KL budget produces surrogate improvement; the remaining $78\%$ is waste. Increasing depth ($E$) does not help; waste grows $11\!\times$ from $E\!=\!2$ to $E\!=\!40$, alignment never exceeds $0.25$, and the returns peak at $E\!\approx\!10$ before collapsing. The dilemma is clear: PPO cannot spend its KL budget efficiently no matter how many epochs it runs.

\begin{table}[h!]
\centering
\caption{\textbf{Signal KL and waste KL vs.\ epoch count.} Fisher decomposition of the final PPO update at varying epoch counts. Hopper, Gymnasium, $(64,64)$ network, 3 seeds. Signal $c^2$ and waste $\|\epsilon\|_F^2$ in $10^{-3}$ KL-nats.}
\label{tab:signal-saturation-body}
\small
\begin{tabular}{@{}lccccccc@{}}
\toprule
\textbf{Epochs $E$} & \textbf{Return} & \textbf{Signal $c^2$} & \textbf{Waste $\|\epsilon\|_F^2$} & \textbf{Total KL} & \textbf{$\alpha$} & \textbf{$\cos_F$} \\
\midrule
2  & $884 \pm 304$  & 1.5  & 6.0   & 7.5   & 0.20 & 0.43 \\
4  & $2222 \pm 529$ & 2.6  & 9.3   & 11.9  & 0.22 & 0.46 \\
6  & $2546 \pm 430$ & 3.4  & 11.5  & 15.0  & 0.23 & 0.47 \\
10 & $\mathbf{3210 \pm 193}$ & 5.3  & 16.3  & 21.5  & 0.24 & 0.49 \\
15 & $3054 \pm 286$ & 7.2  & 22.1  & 29.4  & \textbf{0.25} & 0.49 \\
20 & $1584 \pm 55$  & 8.2  & 31.5  & 39.7  & 0.21 & 0.45 \\
40 & $1296 \pm 292$ & 10.1 & 67.9  & 78.0  & 0.13 & 0.35 \\
\bottomrule
\end{tabular}
\end{table}

\subsection{Motivation for CAPO}
\label{sec:noisy-estimator}

{\em This limitation of PPO motivates replacing depth with width: rather than running one optimizer like PPO for $KE$ epochs, run $K$ experts for $E$ epochs each and average the results; waste partially cancels while signal is preserved.} The remaining question is how to create the $K$ experts. Since collecting samples through environment interactions is the performance bottleneck, we cannot afford to collect an on-policy batch for each expert.

CAPO exploits the following insight. Abstractly, the output of an approximate optimizer like PPO is a \emph{random variable}; different minibatch orderings, initializations, and accumulation paths yield different final policies. 
This randomness is not negligible. \citet{henderson2018deep} show that implementation details and random seeds dominate RL outcomes, and policy gradient estimates are known to be high-variance~\citep{ilyas2020closer}. 

CAPO focuses on one specific source of randomness: the minibatch ordering across PPO's $E$ epochs of SGD. Unlike TRPO, PPO's output depends on the minibatch ordering: given the same batch $\mathcal{B}$ and incumbent $\pi_t$, different shuffles produce different policies. As SGD processes minibatches sequentially, these orthogonal gradient components accumulate into the waste vector. Each minibatch ordering accumulates a different sequence of gradient steps, producing substantially different waste vectors (Section~\ref{sec:diagnostics}, Supplementary~\ref{app:fisher-diagnostics}), yielding different experts.

Section~\ref{sec:method} develops this idea.

\section{CAPO: Consensus Aggregation for Policy Optimization}
\label{sec:method}

Although we focus on PPO in this paper, the idea behind CAPO is quite general and can be used with other trust region methods. 
Given an optimizer $U$ and a current policy $\pi_t$, CAPO collects a batch of data $\mathcal{B}$ under $\pi_t$, and then generates $K$ CAPO experts $\pi^1, ...\pi^K$ with copies of the current incumbent policy parameters $\theta_t$ and Adam optimizer state~\citep{kingma2014adam}, but using different minibatch orderings. 

This approach ensures that any variation across experts is due to the optimization path, not the data. The $K$ expert policies are then aggregated into a consensus $\pi^{agg}$, which serves as the next policy $\pi_{t+1}$. The incumbent's optimizer state is carried forward unchanged (while there are other design choices, this default worked well in our experiments, as shown in Supplementary~\ref{sec:opt-carry}). This process is repeated for multiple generations until convergence or the sample budget is exhausted. Figure~\ref{fig:pipeline} shows this dataflow.

\begin{figure}[t]
\centering
\resizebox{\linewidth}{!}{%
\begin{tikzpicture}[
  stage/.style={
    draw=black!60, rounded corners=4pt, align=center,
    inner sep=6pt, minimum height=10mm,
    font=\small, fill=black!3,
  },
  expbox/.style={
    draw=cExpert!80, rounded corners=3pt, align=center,
    inner sep=5pt, minimum height=9mm, minimum width=40mm,
    font=\small, fill=cExpert!8,
  },
  aggbox/.style={
    draw=cConsensus!80, rounded corners=4pt, align=center,
    inner sep=6pt, minimum height=10mm,
    font=\small, fill=cConsensus!8,
  },
  outbox/.style={
    draw=cOut!80, rounded corners=4pt, align=center,
    inner sep=6pt, minimum height=10mm,
    font=\small, fill=cOut!8,
  },
  arr/.style={-{Latex[length=2.8mm]}, thick, black!65},
  fanarr/.style={-{Latex[length=2.4mm]}, semithick, black!50},
  ann/.style={font=\scriptsize, text=black!55},
]

\node[stage, draw=cIncumbent!80, fill=cIncumbent!12] (pi) {$\pi_t$};

\node[draw=cBatch!60, rounded corners=3pt, fill=cBatch!6,
      align=center, inner sep=4pt, font=\scriptsize,
      right=12mm of pi] (batch)
  {$\mathcal{B} \!\sim\! \pi_t$\\[-1pt] on-policy batch};

\node[expbox, right=28mm of batch, yshift=14mm] (e1)
  {\scriptsize PPO($E$ epochs, seed\,1) $\;\to\; \pi^1$};
\node[expbox, right=28mm of batch, yshift=0mm] (e2)
  {\scriptsize PPO($E$ epochs, seed\,2) $\;\to\; \pi^2$};
\node[right=28mm of batch, yshift=-7.5mm, font=\large, text=black!45] (dots) {$\vdots$};
\node[expbox, right=28mm of batch, yshift=-16mm] (eK)
  {\scriptsize PPO($E$ epochs, seed\,$K$) $\;\to\; \pi^K$};

\node[aggbox, right=20mm of e2, xshift=14mm] (agg)
  {Aggregate\\[-1pt] \scriptsize Param-Avg $\mid$ LogOP};

\node[outbox, right=12mm of agg] (out) {$\pi_{t+1}=\pi^{agg}$};

\draw[arr, densely dashed, black!50] (pi) -- (batch);

\draw[fanarr, densely dashed] (batch.east) -- (e1.west);
\draw[fanarr, densely dashed] (batch.east) -- (e2.west);
\draw[fanarr, densely dashed] (batch.east) -- (eK.west);

\draw[fanarr] (e1.east) -- (agg.west);
\draw[fanarr] (e2.east) -- (agg.west);
\draw[fanarr] (eK.east) -- (agg.west);

\draw[arr] (agg) -- (out);

\draw[decorate, decoration={brace, amplitude=5pt, mirror}, thick, black!45]
  ([yshift=-3mm, xshift=-2mm]eK.south west) -- ([yshift=-3mm, xshift=2mm]eK.south east)
  node[midway, below=5pt, ann] {same data, independent seeds};

\coordinate (loopR) at ([yshift=-35mm]out.south);
\coordinate (loopL) at ([yshift=-35mm]pi.south);
\draw[-{Latex[length=2.8mm]}, thick, black!85]
  (out.south) -- (loopR) -- (loopL) -- (pi.south);
\node[font=\scriptsize, text=black!85] at ($(loopR)!0.5!(loopL)$) [below=1pt] {repeat for $N$ generations};

\node[ann, above=2mm of e1] {$K$ independent optimizer runs};

\end{tikzpicture}%
}
\caption{\textbf{CAPO pipeline.} One on-policy batch is collected from the incumbent $\pi_t$ and fed to $K$ independent PPO copies that differ only in minibatch shuffle order. Expert policies are aggregated into a consensus $\pi_{t+1}$ in parameter space (avg) or distribution space (LogOP).}
\label{fig:pipeline}
\end{figure}
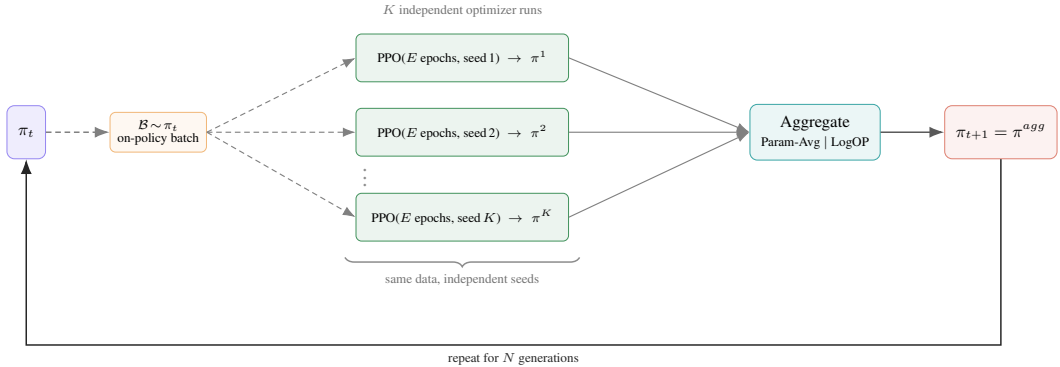

The remaining design choice is the choice of space in which to average. The simplest option is Euclidean parameter space (CAPO-Avg):
\begin{equation}\label{eq:param-avg}
    \theta_{t+1} = \theta_t + \frac{1}{K}\sum_{k=1}^K (\theta^k - \theta_t).
\end{equation}
Alternatively, one can aggregate in distribution space via the logarithmic opinion pool (LogOP)~\citep{genest1986combining, heskes1997logopinion}, which averages expert policies in the natural parameter space:
\begin{equation}\label{eq:logop}
    q(a|s) \;\propto\; \prod_{k=1}^K \pi^k(a|s)^{1/K}.
\end{equation}

For any exponential family, this product stays in the same family: if each expert has natural parameters $\eta^k$, then $q$ has natural parameters $\bar\eta = \frac{1}{K}\sum_k \eta^k$ (the average across the $K$ experts), so $q$ is a proper, closed-form distribution. For diagonal Gaussian experts $\pi^k(\cdot|s) = \mathcal{N}(\mu^k,\mathrm{diag}({\sigma^k}^2))$, this gives a Gaussian consensus with \emph{precision-weighted} parameters:
\begin{equation}\label{eq:precision-weighting}
    \mu_q = \Big(\textstyle\sum_k \tfrac{1}{{\sigma^k}^2}\Big)^{\!-1} \textstyle\sum_k \tfrac{\mu^k}{{\sigma^k}^2}, \qquad
    \sigma_q^{-2} = \tfrac{1}{K}\textstyle\sum_k {\sigma^k}^{-2},
\end{equation}
where all operations are element-wise over action dimensions. An expert with lower variance (higher precision $\sigma^{-2}$) on an action dimension pulls the consensus mean more strongly on that dimension: the LogOP automatically upweights confident experts per dimension, unlike parameter averaging which weights all experts equally. Since $q$ is a per-state analytic distribution, not a neural network, it is distilled into $\pi_{t+1}$ by minimizing $\mathrm{KL}(q \| \pi_{t+1})$ with a line search that rejects any step where $\mathrm{KL}(\pi_t \| \pi_{t+1}) > \delta$. We refer to this variant as \textbf{CAPO} and treat it as the default algorithm.

The two consensus rules obviously coincide when expert variances are identical. Although experts usually stay close to $\pi_t$ within the trust region, there is a small but non-zero gap. Empirically, LogOP outperforms parameter averaging on five of six tasks (Section~\ref{sec:main-results}).

\begin{algorithm}[t]
\caption{CAPO: Consensus Aggregation for Policy Optimization}
\label{alg:capo}
\begin{algorithmic}[1]
\REQUIRE Budget $T$, experts $K$, PPO epochs $E$
\STATE Initialize policy $\pi_0$, value function $V_0$, optimizer state $\phi_0$
\STATE Optional warmup: standard PPO until return exceeds floor
\FOR{generation $g = 0, 1, 2, \ldots$ until budget exhausted}
    \STATE Collect batch $\mathcal{B} \sim \pi_t$ ($C$ frames); compute $\hat{A}$ via GAE($\lambda$)
    \FOR{$k = 1, \ldots, K$ \textbf{in parallel}}
        \STATE $\theta^k \leftarrow \theta_t$, $\phi^k \leftarrow \phi_t$ \hfill\COMMENT{Copy incumbent}
        \STATE Run $E$ PPO epochs with shuffle seed $k$ $\to \theta^k$
    \ENDFOR
    \STATE $\theta_{t+1} \leftarrow \mathrm{Aggregate}(\theta^1, \ldots, \theta^K)$ \hfill\COMMENT{Avg or LogOP}
    \STATE $\phi_{t+1} \leftarrow \phi_t$ \hfill\COMMENT{Carry incumbent optimizer state}
    \STATE Update value function $V$ on $\mathcal{B}$ \hfill\COMMENT{Trained once, not replicated}
\ENDFOR
\end{algorithmic}
\end{algorithm}

Pseudocode for CAPO is provided in Algorithm~\ref{alg:capo}. 
The computational overhead is $K\!\times$ gradient computation per generation, no additional environment interaction, and $1\!\times$ inference cost (Table~\ref{tab:compute}).

\begin{table}[h!]
\centering
\caption{\textbf{Compute comparison.} CAPO trades parallel gradient computation for reduced update variance, without additional environment frames or inference cost.}
\label{tab:compute}
\small
\begin{tabular}{@{}lcccc@{}}
\toprule
\textbf{Method} & \textbf{Env.\ Frames} & \textbf{Grad.\ Steps} & \textbf{Inference} & \textbf{Update Quality} \\
\midrule
PPO & $C$ & $EM$ & $1\times$ & Baseline \\
PPO-$K\!\times$ (more epochs) & $C$ & $K \!\cdot\! EM$ & $1\times$ & \emph{Degrades} \\
PPO-SWA & $C$ & $EM$ & $1\times$ & Moderate \\
Best-of-$K$ & $C$ & $K \!\cdot\! EM$ & $1\times$ & Selection only \\
Ensemble Methods & $C$ & $K \!\cdot\! EM$ & $K\times$ & Moderate \\
\textbf{CAPO (Ours)} & $C$ & $K \!\cdot\! EM$ & $1\times$ & \textbf{Best} \\
\bottomrule
\end{tabular}
\end{table}




\paragraph{Budget allocation.}
\label{sec:budget} 
A fraction $f_w$ of the total step budget $C$ is spent on PPO warmup to initialize the incumbent. The remaining $(1 {-} f_w)C$ steps are divided across $G$ generations, each collecting one on-policy batch of $B$ frames and running $K$ expert updates in parallel on that batch. Because experts share the same batch, the per-expert sample budget is $(1 {-} f_w)C / G$, identical to single-expert PPO, while gradient compute scales as $K\!\times$. The warmup fraction trades initial policy quality against the number of aggregation generations; we ablate this tradeoff in Section~\ref{sec:warmup-ablation}, and default to 10\% in our experiments. 

\paragraph{Expert diversity.}
\label{sec:diversity} 
In our current configuration, the sole source of expert diversity is minibatch ordering: each expert receives a different shuffle seed, producing different minibatch sequences through the same data. Optionally, experts can also use different GAE $\lambda$ values, giving each expert a different advantage estimate and thus a different optimization target. We leave this for future work.



\section{Consensus Improvement Theorem}
\label{sec:consensus-improvement}
 
Theorem~\ref{thm:consensus} formalizes the benefit of consensus aggregation. We state it in natural parameter space $\eta$ where the KL bound is exact (exponential families generalize Gaussians). 

\begin{theorem}[Consensus improvement in natural parameter space]\label{thm:consensus}
Let $\pi_t$ belong to a minimal exponential family with natural parameters $\eta_t$ and log-partition function $A(\eta)$. Let $\eta^1, \ldots, \eta^K$ be expert natural parameters and $\bar{\eta} = \frac{1}{K}\sum_k \eta^k$. Then:
\begin{enumerate}[label=(\alph*),nosep]
  \item \textbf{Averaging preserves expected improvement.}\;
        The linearized surrogate $\hat{L}_{\mathrm{lin}}(\eta) = g_\eta^\top\!(\eta - \eta_t)$, where $g_\eta = \nabla_\eta L^{\mathrm{surr}}|_{\eta_t}$, is linear in $\eta$. Therefore $\hat{L}_{\mathrm{lin}}(\bar{\eta}) = \frac{1}{K}\sum_k \hat{L}_{\mathrm{lin}}(\eta^k)$.

  \item \textbf{Averaging reduces KL cost (exactly).}\;
        $\mathrm{KL}(\pi_t \| \pi_\eta) = D_A(\eta, \eta_t)$ is convex in $\eta$ since $\nabla^2 A(\eta) = F(\eta) \succeq 0$. Jensen's inequality gives:
        \[
          \mathrm{KL}(\pi_t \| \pi_{\bar\eta}) \;\leq\; \frac{1}{K}\sum_k \mathrm{KL}(\pi_t \| \pi_{\eta^k}).
        \]
        Unlike the quadratic Fisher approximation $\frac{1}{2}\Delta^\top\!F\,\Delta$, this is exact.

  \item \textbf{Better improvement per unit of KL.}\;
        For any $\lambda > 0$, the KL-penalized surrogate $J_\lambda(\eta) = \hat{L}_{\mathrm{lin}}(\eta) - \lambda\,\mathrm{KL}(\pi_t \| \pi_\eta)$ satisfies
        \[
          J_\lambda(\bar{\eta})
          \;\geq\; \frac{1}{K}\sum_k J_\lambda(\eta^k).
        \]

  \item \textbf{Trust region compliance.}\;
        The sublevel set $\mathcal{C}_\delta = \{\eta : \mathrm{KL}(\pi_t \| \pi_\eta) \leq \delta\}$ is convex. If each $\eta_k \in \mathcal{C}_\delta$, then $\bar\eta \in \mathcal{C}_\delta$.
\end{enumerate}
\end{theorem}

\begin{proof}
\noindent\textbf{(a)} $\hat{L}_{\mathrm{lin}}(\eta) = g_\eta^\top(\eta - \eta_t)$ is linear in $\eta$, so $\hat{L}_{\mathrm{lin}}(\bar\eta) = \frac{1}{K}\sum_k \hat{L}_{\mathrm{lin}}(\eta_k)$.

\noindent\textbf{(b)} $\mathrm{KL}(\pi_t \| \pi_\eta) = D_A(\eta, \eta_t)$ is a Bregman divergence of the convex log-partition $A(\eta)$. Jensen's inequality gives $D_A(\bar\eta, \eta_t) \leq \frac{1}{K}\sum_k D_A(\eta_k, \eta_t)$.

\noindent\textbf{(c)} $J_\lambda(\bar\eta) = \hat{L}_{\mathrm{lin}}(\bar\eta) - \lambda\,\mathrm{KL}(\pi_t \| \pi_{\bar\eta}) = \frac{1}{K}\sum_k \hat{L}_{\mathrm{lin}}(\eta_k) - \lambda\,\mathrm{KL}(\pi_t \| \pi_{\bar\eta}) \geq \frac{1}{K}\sum_k J_\lambda(\eta_k)$, using (a) and (b).

\noindent\textbf{(d)} $\mathcal{C}_\delta = \{\eta : D_A(\eta, \eta_t) \leq \delta\}$ is convex (sublevel set of a convex function), so $\eta_k \in \mathcal{C}_\delta \;\forall k \implies \bar\eta \in \mathcal{C}_\delta$.\qedhere
\end{proof}

\begin{remark}[Connection to the Fisher decomposition]\label{rem:fisher-connection}
In parameter space, the quadratic approximation $\mathrm{KL} \approx \frac{1}{2}\Delta^\top\!F\,\Delta$ yields $J_\lambda(\bar\Delta) - \frac{1}{K}\sum_k J_\lambda(\Delta^k) = \frac{\lambda}{2}[\mathrm{Var}(c^k) + \frac{1}{K}\sum_k\|\epsilon^k - \bar\epsilon\|_F^2]$. Signal variance is empirically small, so waste reduction drives the gain. Alignment $\alpha(\bar\Delta) > \frac{1}{K}\sum_k \alpha(\Delta^k)$ follows when $\mathrm{Var}(c^k) \ll \bar{c}^2$.
\end{remark}

The following corollary summarizes the key takeaway. 

\begin{corollary}[Consensus is a better policy improvement step]\label{cor:better-step}
Let experts $\eta^1, \ldots, \eta^K$ each lie within the trust region ($\mathrm{KL}(\pi_t \| \pi_{\eta^k}) \leq \delta$). The natural-parameter consensus $\bar{\eta} = \frac{1}{K}\sum_k \eta^k$ simultaneously:
\begin{enumerate}[label=(\roman*),nosep]
    \item achieves higher KL-penalized surrogate than the mean expert: $J_\lambda(\bar{\eta}) \geq \frac{1}{K}\sum_k J_\lambda(\eta^k)$;
    \item stays within the trust region: $\mathrm{KL}(\pi_t \| \pi_{\bar\eta}) \leq \frac{1}{K}\sum_k \mathrm{KL}(\pi_t \| \pi_{\eta^k}) \leq \delta$.
\end{enumerate}
For CAPO-Avg, these guarantees hold approximately in parameter space. For CAPO, they hold exactly for the analytic log opinion pool barycenter $q$; the distilled neural policy $\pi_{t+1}$ inherits them up to distillation error.
\end{corollary}

\section{Experiments}
\label{sec:experiments}

The results in this section show that the theoretical benefits of consensus aggregation translate to improved returns on continuous control benchmarks.

\subsection{Setup and Environments}

We implement CAPO in Gymnasium MuJoCo-v4 using $(64, 64)$ networks following CleanRL~\citep{huang2022cleanrl}, with $K\!=\!4$ experts by default (ablated in Supplementary~\ref{app:k-sweep}). We evaluate on six continuous control tasks (Table~\ref{tab:envs}), reporting mean episodic return over 20 evaluation episodes.

\begin{table}[h!]
\centering
\caption{\textbf{Evaluation environments.} Task type, state and action space dimensions. All tasks use Gymnasium MuJoCo-v4.}
\label{tab:envs}
\small
\begin{tabular}{@{}llccc@{}}
\toprule
Environment & Type & $\dim(\mathcal{S})$ & $\dim(\mathcal{A})$ \\
\midrule
Hopper & Hopping locomotion & 11 & 3 \\
HalfCheetah & Locomotion & 17 & 6 \\
Walker2d & Locomotion & 17 & 6 \\
Ant & Multi-legged locomotion & 27 & 8 \\
Humanoid & High-dim locomotion & 376 & 17 \\
HumanoidStandup & High-dim balance & 376 & 17 \\
\bottomrule
\end{tabular}
\end{table}

We compare the two CAPO aggregation operators, \textbf{CAPO} and \textbf{CAPO-Avg} against five baselines including \textbf{PPO}~\citep{schulman2017ppo} and \textbf{TRPO}~\citep{schulman2015trpo}. The other three baselines are \textbf{PPO-$K\!\times$}, which runs PPO with $KE$ epochs on the same batch, matching CAPO's total gradient cost but spending it sequentially, \textbf{Best-of-$K$}, which trains $K$ independent PPO copies and selects the one with the highest surrogate, isolating selection from averaging, and \textbf{PPO-SWA} (stochastic weight averaging~\citep{izmailov2018averaging} along the PPO trajectory).

\subsection{Main Results}
\label{sec:main-results}

\begin{table}[h!]
\centering
\caption{\textbf{Returns on Gymnasium} (mean $\pm$ SE, 8 seeds, $(64,64)$ MLP, 1M steps; Humanoid at 4M). $^\dagger$Compute-matched ($K\!\times$ gradient budget). Bold = best per task.}
\label{tab:main-results}
\small
\setlength{\tabcolsep}{3pt}
\begin{tabular}{@{}l cccccc@{}}
\toprule
\textbf{Method}
  & \textbf{Hopper}
  & \textbf{HalfCheetah}
  & \textbf{Walker2d}
  & \textbf{Ant}
  & \textbf{Humanoid}
  & \textbf{HumanoidStandup} \\
\midrule
PPO & $2711\!\pm\!268$ & $1604\!\pm\!28$ & $2518\!\pm\!280$ & $2142\!\pm\!97$ & $739\!\pm\!39$ & $143\text{k}\!\pm\!3\text{k}$ \\
PPO-$K\!\times^\dagger$ & $1210\!\pm\!304$ & $1532\!\pm\!91$ & $1432\!\pm\!337$ & $232\!\pm\!21$ & $429\!\pm\!11$ & $104\text{k}\!\pm\!5\text{k}$ \\
Best-of-$K^\dagger$ & $2255\!\pm\!274$ & $1940\!\pm\!294$ & $3770\!\pm\!341$ & $1664\!\pm\!241$ & $716\!\pm\!41$ & $145\text{k}\!\pm\!4\text{k}$ \\
PPO-SWA & $2240\!\pm\!307$ & $1261\!\pm\!431$ & $1010\!\pm\!127$ & $943\!\pm\!88$ & $750\!\pm\!55$ & $121\text{k}\!\pm\!6\text{k}$ \\
TRPO & $2583\!\pm\!300$ & $1629\!\pm\!27$ & $3441\!\pm\!303$ & $1029\!\pm\!102$ & $3730\!\pm\!339$ & $134\text{k}\!\pm\!4\text{k}$ \\
\midrule
CAPO-Avg$^\dagger$ & $\mathbf{3193\!\pm\!93}$ & $1967\!\pm\!260$ & $3223\!\pm\!354$ & $1738\!\pm\!185$ & $696\!\pm\!28$ & $141\text{k}\!\pm\!5\text{k}$ \\
\textbf{CAPO}$^\dagger$ & $2397\!\pm\!286$ & $\mathbf{2737\!\pm\!373}$ & $\mathbf{3871\!\pm\!275}$ & $\mathbf{2328\!\pm\!194}$ & $\mathbf{6367\!\pm\!196}$ & $\mathbf{149\text{k}\!\pm\!3\text{k}}$ \\
\bottomrule
\end{tabular}
\end{table}

\begin{figure}[h]
\centering
\includegraphics[width=\linewidth]{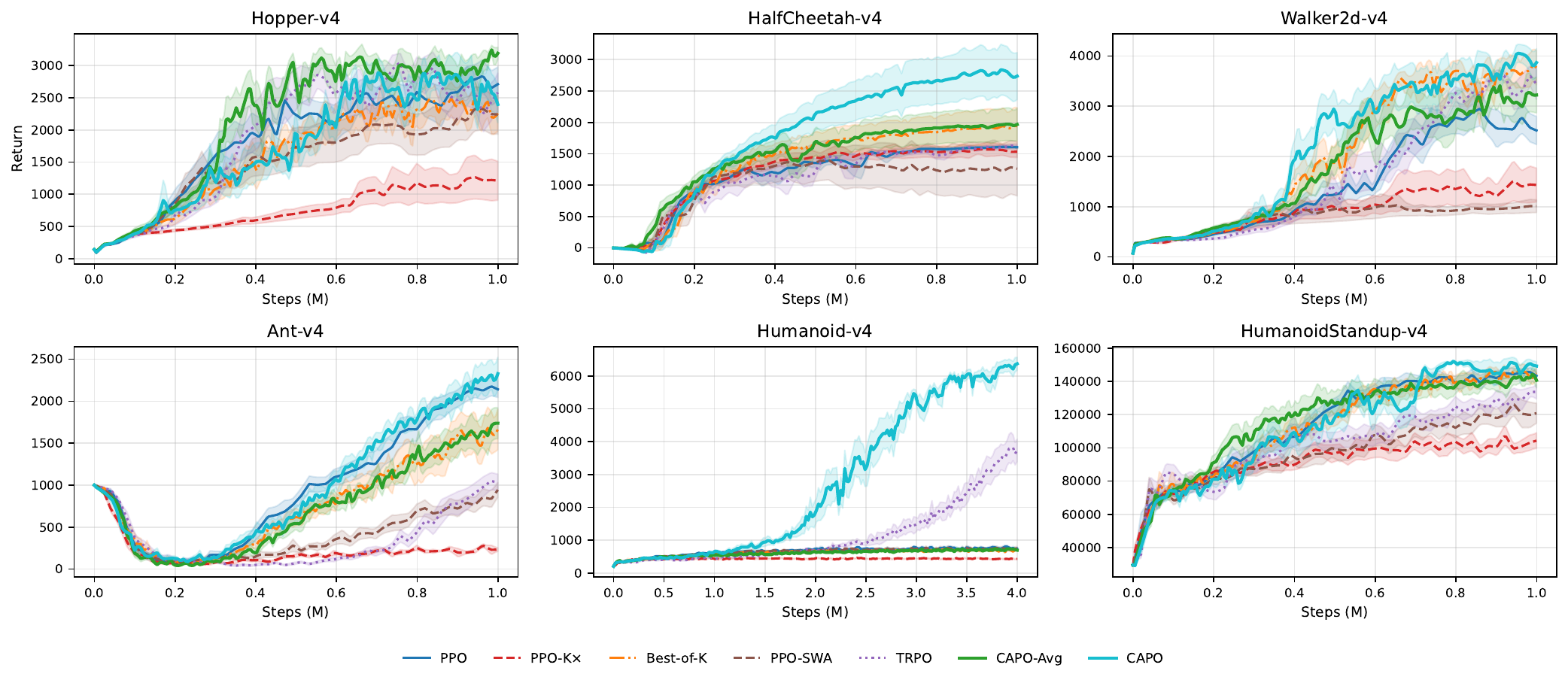}
\caption{\textbf{Gymnasium learning curves} (8 seeds, shaded $\pm 1$ SE). CAPO leads on all tasks except Hopper, which is dominated by CAPO-Avg. PPO-$K\!\times$ collapses on Ant ($7\!\times$ below PPO) and Walker2d, validating the optimization-depth dilemma.}
\label{fig:main-curves}
\end{figure}

Figure~\ref{fig:main-curves} shows learning curves for all methods; Table~\ref{tab:main-results} shows final evaluation results. CAPO beats baselines on all tasks except Hopper, which CAPO-Avg dominates. The largest gains are on HalfCheetah ($2737$, $+71\%$ over PPO) and Walker2d ($3871$, $+54\%$).

On Ant, CAPO ($2328$) exceeds PPO ($2142$). On Hopper, CAPO-Avg ($3193$) achieves the highest return while CAPO ($2397$) is comparable to PPO ($2463$); this low-dimensional task limits the benefit of precision-weighting. On Humanoid, CAPO ($6367$) achieves $8.6\!\times$ the return of PPO ($739$); TRPO ($3730$) also benefits strongly, while CAPO-Avg ($696$) remains near PPO parity, demonstrating that LogOP's per-dimension precision-weighting is beneficial in high-dimensional action spaces over longer horizons. On HumanoidStandup, CAPO outperforms all baselines.

PPO-$K\!\times$ degrades on every task, most catastrophically on Ant ($232$, $9\!\times$ below PPO), confirming the depth dilemma. PPO-SWA also degrades, confirming that trajectory averaging is insufficient. Best-of-$K$ improves over PPO on some tasks but inherits the selected expert's full waste. Wall clock overhead is modest: ${\sim}0$--$40\%$ for CAPO-Avg, $19$--$72\%$ for CAPO (Supplementary~\ref{app:timing}).

\subsection{Fisher Decomposition Diagnostics}
\label{sec:diagnostics}

We validate the empirical behavior predicted by the Fisher decomposition. 

\begin{figure}[h!]
    \includegraphics[width=\linewidth]{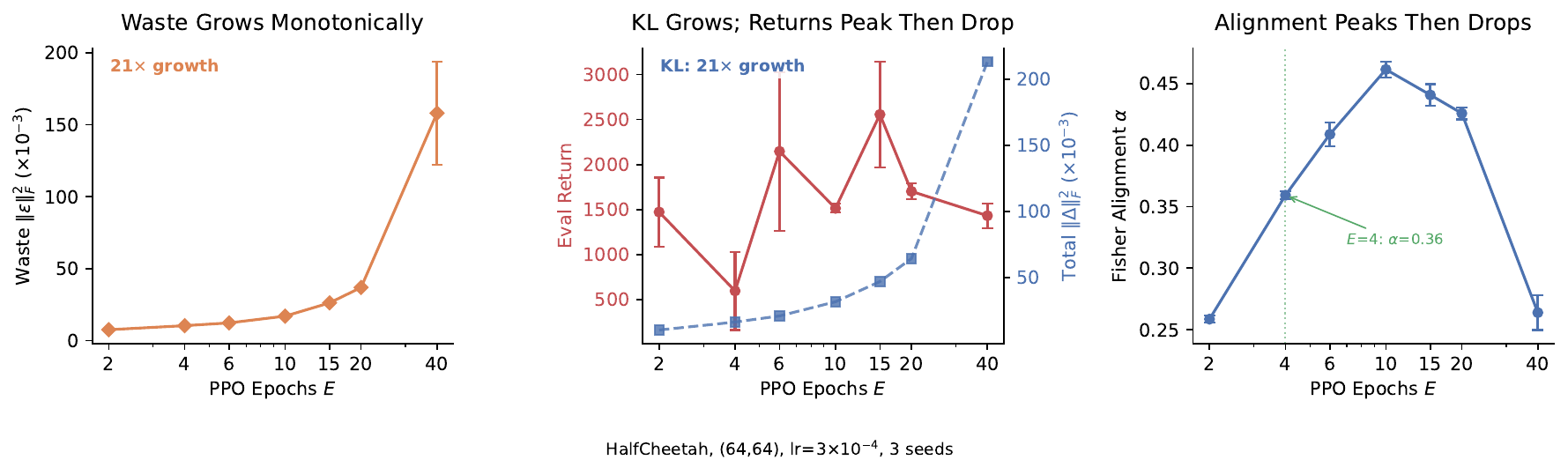}
    \caption{\textbf{Fisher diagnostics for the last PPO update vs.\ epoch count} (HalfCheetah, $(64,64)$ network, 3 seeds). Left: waste $\|\epsilon\|_F^2$ and signal $c^2$ on log scale. Center: total KL and return. Right: Fisher alignment. Waste grows $21\!\times$ from $E\!=\!2$ to $E\!=\!40$; returns peak at $E\!\approx\!6$--$15$ then decline; alignment $\alpha$ peaks at $0.46$ then drops to $0.26$.}
    \label{fig:fisher-vs-epochs}
\end{figure}

\textbf{(i) KL waste dominates as number of epochs grow}: Section~\ref{sec:depth-dilemma} showed this on Hopper (Table~\ref{tab:signal-saturation-body}); Figure~\ref{fig:fisher-vs-epochs} confirms the pattern on HalfCheetah. At $E\!=\!4$, waste accounts for $64\%$ of the total KL update ($\alpha = 0.36$). Waste grows $21\!\times$ from $E\!=\!2$ to $E\!=\!40$ without plateauing; signal also grows but slower. Fisher alignment $\alpha$ peaks at $E\!\approx\!10$ ($\alpha = 0.46$) then drops to $0.26$ at $E\!=\!40$; the returns peak early ($E\!\approx\!6$--$15$) and decline as total KL overshoots the trust region.

\textbf{(ii) Waste is less correlated than signal} (Table~\ref{tab:signal-waste-corr}): signal correlation $\rho_c > 0.99$ while waste correlation $\rho_\epsilon \in [0.77, 0.97]$. 

\textbf{(iii) Averaging directly reduces waste} (Table~\ref{tab:waste-reduction-direct}): CAPO-Avg reduces waste by $2$--$17\%$, with the largest reduction on Humanoid. On low-dimensional tasks, LogOP amplifies waste ($5$--$9\!\times$) but improves returns nevertheless. On high-dimensional tasks (Ant, Humanoid) it reduces waste, confirming that LogOP's gains come from precision weighting. 

Full tables and methodology for (ii) and (iii) in Supplementary~\ref{app:fisher-diagnostics}. 

\subsection{Ablations}
\label{sec:ablations}

We ablate four design choices in Supplementary~\ref{app:ablations}: \textbf{optimizer state carry} (Section~\ref{sec:opt-carry}), \textbf{clip threshold $\epsilon$} (Section~\ref{app:clip}), KL early stopping (Section~\ref{app:kl-early-stop}), and \textbf{warm-up budget} (Section~\ref{sec:warmup-ablation}); and \textbf{width scaling} (Section~\ref{app:k-sweep}), varying $K$ at fixed epoch count $E$. The key findings are:

\begin{enumerate}[nosep]
  \item \textbf{Width scales; depth does not.} Increasing $K$ at fixed $E$ improves returns monotonically (CAPO) or up to $K\!=\!8$ (CAPO-Avg), whereas increasing $E$ past ${\sim}10$ degrades all methods.
  \item \textbf{CAPO is robust to trust region parameters.} CAPO maintains strong performance across a range of clipping and KL early stopping thresholds. PPO's own performance is sensitive to these choices (e.g., dropping $3\!\times$ on Walker2d when $\epsilon$ deviates from $0.2$), while CAPO tolerates a wider operating range.
  \item \textbf{Warmup and optimizer carry are secondary.} Performance varies modestly across warmup fractions ($0$--$20\%$) and carry modes; the default $10\%$ warmup with incumbent carry is a reasonable choice across tasks.
\end{enumerate}

\section{Related Work}
\label{sec:related}

CAPO draws on several research threads, primarily weight averaging in optimization, and ensemble and population-based methods in RL.

\paragraph{Weight averaging in optimization.}
SWA~\citep{izmailov2018averaging}, Lookahead~\citep{zhang2019lookahead}, and \citet{nikishin2018improving} average along a single optimization trajectory; CAPO averages across $K$ parallel trajectories, which cancels path-dependent waste that temporal averaging preserves (PPO-SWA degrades in performance, as shown in Table~\ref{tab:main-results}). SWAP~\citep{gupta2020stochastic} averages parallel workers on the same data with different minibatch orderings, the closest supervised-learning analog to CAPO, but lacks trust region structure or Fisher-geometric analysis. Model soups~\citep{wortsman2022model}, WARP~\citep{rame2024warm}, and Rewarded Soups~\citep{rame2023rewarded} average across hyperparameters or reward models; CAPO averages across optimization paths from the \emph{same} hyperparameters and objective, isolating optimizer noise as the diversity source. One-shot averaging~\citep{zinkevich2010parallelized}, local SGD~\citep{stich2019local}, and DiLoCo~\citep{douillard2023diloco} average parallel workers with heterogeneous data; CAPO operates on homogeneous data where all diversity comes from the optimizer, making the signal--waste decomposition tractable.

\paragraph{Ensemble and population methods in RL.}
EPPO~\citep{yang2022eppo} maintains $K$ policies at inference time ($K\!\times$ inference cost); CAPO aggregates at optimization time and deploys a single policy ($1\!\times$ inference). CEM-RL~\citep{pourchot2019cemrl}, ERL~\citep{khadka2018evolutionary}, and EPO~\citep{mustafaoglu2025epo} combine evolutionary perturbations with gradients; CAPO requires no random perturbations, deriving all diversity from stochastic optimization paths. \citet{qiu2025es} scale ES to multi-billion-parameter LLM fine-tuning via many simple perturbation steps; CAPO takes fewer, gradient-informed steps and exploits curvature via the natural gradient. SVPG~\citep{liu2017stein} maintains diverse particles via repulsive kernels; CAPO aggregates into a single consensus without explicit repulsion.

\section{Discussion and Future Work}
\label{sec:discussion}

CAPO exploits the structure of optimizer noise in trust region optimization: signal saturates early while waste accumulates with additional epochs, so averaging parallel copies cancels waste while preserving signal. The principle generalizes: in any setting where iterative refinement on fixed data produces diminishing-return noise, width-first strategies may offer better variance-bias tradeoffs.

\paragraph{Practical tips for practitioners.}  For PPO, returns peak at an epoch count of 10, and degrade significantly by $E$=20. Implementing KL early stopping (ablated in~\ref{sec:ablations}) is recommended. $K$=8 yields the best returns, but $K$=4 was used in our experiments to handle the efficiency/performance trade-off. 

\paragraph{Limitations.}
$K\!\times$ gradient cost is the primary overhead, though expert training is embarrassingly parallel and wall-clock time scales sublinearly with $K$. All experiments use continuous control with Gaussian policies; the theory applies to exponential families generally, but we have no empirical evidence for discrete actions, vision, or language tasks.

\paragraph{Future work.}
Future work could explore principled methods for generating diverse experts for increased waste reduction.
We also explored richer diversity sources beyond minibatch ordering, such as Bayesian bootstrap reweighting of advantages and diverse GAE lambda values, which increase expert disagreement but introduce estimator heterogeneity not covered by Theorem~\ref{thm:consensus}. The signal--waste decomposition suggests CAPO could improve LLM fine-tuning, where  optimizer noise compounds across long sequences. 

\subsubsection*{Broader Impact Statement}
\label{sec:broaderImpact}
CAPO reduces costly environment interactions, potentially accelerating robotics and simulation-based RL research.

\appendix


\bibliography{main}
\bibliographystyle{rlj}

\beginSupplementaryMaterials

\section{Fisher Decomposition Diagnostics}
\label{app:fisher-diagnostics}

This appendix extends the empirical analysis of Sections~\ref{sec:depth-dilemma} and~\ref{sec:diagnostics}, and presents full Fisher decomposition diagnostics: signal--waste correlation, and direct waste reduction measurement. 

\subsection{Signal/Waste Correlation}
\label{app:correlation}

Table~\ref{tab:signal-waste-corr} breaks down the signal and waste correlation at the final generation for CAPO-Avg.

\begin{table}[h!]
\centering
\caption{\textbf{Signal vs.\ waste correlation at the final generation} ($K\!=\!4$), CAPO-Avg. Gymnasium, $(64,64)$ network, 5 seeds.}
\label{tab:signal-waste-corr}
\small
\begin{tabular}{@{}lccccc@{}}
\toprule
\textbf{Environment}
  & $\rho_c$ (signal)
  & $\rho_\epsilon$ (waste)
  & $\rho_{\mathrm{total}}$
  & $\bar{c}/\bar{c}_k$
  \\
\midrule
HalfCheetah & 0.999 & 0.974 & 0.976 & 0.999 \\
Walker2d & 0.999 & 0.961 & 0.966 & 0.999  \\
Hopper & 0.993 & 0.934 & 0.939 & 0.995 \\
Ant & 1.000 & 0.969 & 0.973 & 1.000 \\
Humanoid & 1.000 & 0.774 & 0.837 & 1.000 \\
\bottomrule
\end{tabular}
\end{table}

Waste correlation $\rho_\epsilon$ ranges from $0.77$ (Humanoid) to $0.97$ (HalfCheetah).

\subsection{Waste Reduction}
\label{app:waste-reduction}

The correlation analysis above predicts waste cancellation; we now measure it directly. We decompose the consensus update $\Delta_{\mathrm{cons}}$ into signal and waste using the same Fisher projection, and report the ratio $\|\epsilon_{\mathrm{cons}}\|_F^2 \big/ \frac{1}{K}\sum_k\|\epsilon_k\|_F^2$ in Table~\ref{tab:waste-reduction-direct}. A ratio below 1 means the consensus has less waste than the average expert; above 1 means waste was amplified.

\begin{table}[h]
\centering
\caption{\textbf{Direct waste reduction: consensus vs.\ PPO} ($K\!=\!4$, $E=10$). Ratio $\|\bar\epsilon\|_F^2 / \frac{1}{K}\sum_k\|\epsilon_k\|_F^2$ measures how much waste the consensus retains relative to PPO. Values ${<}\,1$: waste reduced; values ${>}\,1$: waste amplified. $d_s$: observation dimension; $d_a$: action dimension. Gymnasium, $(64,64)$ network, 5 seeds.}
\label{tab:waste-reduction-direct}
\small
\begin{tabular}{@{}llccccc@{}}
\toprule
\textbf{Environment} & \textbf{Method} & $d_s$ & $d_a$
  & $\frac{1}{K}\!\sum_k\!\|\epsilon_k\|_F^2$
  & $\|\bar\epsilon\|_F^2$
  & \textbf{Ratio} \\
\midrule
Hopper & CAPO-Avg  & 11 & 3 & $7.8\!\times\!10^{-4}$ & $7.4\!\times\!10^{-4}$ & $0.950 \pm 0.015$ \\
Hopper & CAPO      & 11 & 3 & $1.3\!\times\!10^{-3}$ & $6.5\!\times\!10^{-3}$ & $5.68 \pm 1.31$ \\
\addlinespace
HalfCheetah & CAPO-Avg & 17 & 6 & $2.7\!\times\!10^{-3}$ & $2.7\!\times\!10^{-3}$ & $0.980 \pm 0.003$ \\
HalfCheetah & CAPO     & 17 & 6 & $3.5\!\times\!10^{-3}$ & $2.4\!\times\!10^{-2}$ & $9.27 \pm 2.46$ \\
\addlinespace
Walker2d & CAPO-Avg & 17 & 6 & $1.1\!\times\!10^{-3}$ & $1.1\!\times\!10^{-3}$ & $0.971 \pm 0.005$ \\
Walker2d & CAPO     & 17 & 6 & $9.7\!\times\!10^{-4}$ & $7.6\!\times\!10^{-3}$ & $8.37 \pm 1.13$ \\
\addlinespace
Ant & CAPO-Avg  & 27 & 8 & $8.8\!\times\!10^{-2}$ & $8.0\!\times\!10^{-2}$ & $0.914 \pm 0.001$ \\
Ant & CAPO      & 27 & 8 & $6.9\!\times\!10^{-2}$ & $3.7\!\times\!10^{-2}$ & $\mathbf{0.668 \pm 0.004}$ \\
\addlinespace
Humanoid  & CAPO-Avg  & 376 & 17 & $0.342$ & $0.285$ & $\mathbf{0.832 \pm 0.002}$ \\
Humanoid  & CAPO      & 376 & 17 & $0.058$ & $0.032$ & $\mathbf{0.559 \pm 0.022}$ \\
\bottomrule
\end{tabular}
\end{table}

CAPO-Avg reduces waste on every environment, with reduction scaling with parameter count: $2\%$ on HalfCheetah/Walker2d ($n_\theta\!=\!17$), $5\%$ on Hopper ($n_\theta\!=\!11$), and $17\%$ on Humanoid ($n_\theta\!=\!376$). This directly confirms the theoretical prediction that averaging cancels waste while preserving signal.

CAPO (LogOP) demonstrates a striking pattern: it amplifies waste $5$--$10\!\times$ on low-dimensional tasks ($n_\theta \leq 17$), but achieves $33\%$ waste reduction on Ant ($d_s\!=\!27$) and $46\%$ waste reduction on Humanoid ($n_\theta\!=\!376$). This illustrates how the benefits from precision-weighting scale with the dimensionality of the task.

\section{Experimental Details}
\label{app:details}

This section presents all hyperparameters used in the experiments, and reports a comparison of wall-clock time for all methods.

\subsection{Hyperparameters}
\label{app:hyperparameters}

\subsubsection{CAPO Hyperparameters}

\begin{center}
\begin{tabular}{@{}lll@{}}
\toprule
Parameter & Default & Description \\
\midrule
$K$ (n\_experts) & 4 (see Section~\ref{app:k-sweep}) & Number of parallel expert policies per generation \\
Carry policy & incumbent & Optimizer state carryover between generations \\
\bottomrule
\end{tabular}
\end{center}

\subsubsection{PPO Hyperparameters}

\begin{center}
\small
\begin{tabular}{@{}lll@{}}
\toprule
Parameter & Value & Description \\
\midrule
Hidden layers & $(64, 64)$ & Policy and value network \\
Activation & tanh & Hidden layer activation \\
Initialization & Orthogonal & Weight init method \\
Std parameterization & clip & Log-std to std conversion \\
Learning rate & $3\!\times\!10^{-4}$ & Adam LR (policy and value) \\
Adam $\epsilon$ & $10^{-5}$ & Adam epsilon \\
LR annealing & linear $\to 0$ & Over total optimization steps \\
Discount $\gamma$ & 0.99 & Return discount factor \\
GAE $\lambda$ & 0.95 & Generalized advantage estimation \\
Clip $\epsilon$ & 0.2 & PPO clipping threshold \\
Entropy coef. & 0.0 & Entropy bonus coefficient \\
VF coef. & 0.5 & Value loss coefficient \\
Max grad norm & 0.5 & Global gradient norm clipping \\
PPO epochs $E$ & 10 & Gradient passes per rollout \\
Minibatches & 32 & Per epoch \\
Rollout horizon $H$ & 512 & Steps per environment \\
Parallel envs & 8 & Training environments \\
\bottomrule
\end{tabular}
\end{center}

\subsubsection{Base Optimizer Hyperparameters}

\begin{center}
\begin{tabular}{@{}lll@{}}
\toprule
Parameter & Default & Description \\
\midrule
base\_optimizer & ppo & $\in \{\text{ppo}, \text{trpo}\}$ \\
trpo\_damping & 0.1 & Fisher-vector product damping for CG (TRPO only) \\
trpo\_cg\_iters & 10 & Max conjugate gradient iterations (TRPO only) \\
\bottomrule
\end{tabular}
\end{center}

\subsection{Wall-Clock Timing}
\label{app:timing}

APO's $K\!\times$ gradient cost is the primary overhead. Table~\ref{tab:wall-clock} reports wall-clock times on Gymnasium. CAPO-Avg adds ${\sim}0$--$40\%$ wall-clock overhead despite $4\!\times$ gradient cost, because CAPO's fewer generations amortize overhead (JIT recompilation, env resets, evaluation). CAPO adds $19$--$72\%$ due to the distillation step. Notably, PPO-$K\!\times$ is often \emph{slower} ($21$--$65\%$ overhead) because its sequential epochs cannot be parallelized. 

\begin{table}[h]
\centering
\caption{Wall-clock time on Gymnasium, 8 seeds. PPO row: absolute seconds (mean $\pm$ SE). Other rows: ratio to PPO ($\times$). Five tasks at 1M steps; Humanoid at 4M$^\dagger$.}
\label{tab:wall-clock}
\scriptsize
\setlength{\tabcolsep}{3pt}
\begin{tabular}{@{}lcccccc@{}}
\toprule
& \textbf{HalfCheetah} & \textbf{Walker2d} & \textbf{Ant} & \textbf{Hopper} & \textbf{Humanoid}$^\dagger$ & \textbf{HumanoidStandup} \\
\midrule
PPO & $1331\!\pm\!155$ & $1376\!\pm\!179$ & $1563\!\pm\!247$ & $1027\!\pm\!8$ & $6756\!\pm\!837$ & $1629\!\pm\!82$ \\
PPO-$K\!\times$ & 1.37$\times$ & 1.36$\times$ & 1.38$\times$ & 1.41$\times$ & 1.21$\times$ & 1.65$\times$ \\
TRPO & 0.98$\times$ & 0.99$\times$ & 1.03$\times$ & 1.01$\times$ & 1.02$\times$ & 0.95$\times$ \\
PPO-SWA & 0.97$\times$ & 0.99$\times$ & 1.04$\times$ & 1.00$\times$ & 0.92$\times$ & 1.06$\times$ \\
Best-of-$K$ & 1.36$\times$ & 1.31$\times$ & 1.18$\times$ & 1.67$\times$ & 1.06$\times$ & 1.01$\times$ \\
CAPO-Avg & 1.12$\times$ & 1.16$\times$ & 0.96$\times$ & 1.40$\times$ & 0.99$\times$ & 1.00$\times$ \\
CAPO & 1.54$\times$ & 1.72$\times$ & 1.39$\times$ & 1.42$\times$ & 1.19$\times$ & 1.28$\times$ \\
\bottomrule
\end{tabular}
\end{table}

\section{Ablations and Scaling}
\label{app:ablations}

\subsection{Optimizer State Carry}
\label{sec:opt-carry}

\begin{table}[h!]
\centering
\caption{Optimizer state carry ablation: CAPO-Avg and CAPO across three Gymnasium environments (1M steps, 3 seeds, mean $\pm$ SE). Bold indicates the best carry mode per row.}
\label{tab:opt-carry}
\small
\setlength{\tabcolsep}{4pt}
\begin{tabular}{@{}llcccc@{}}
\toprule
\textbf{Method} & \textbf{Env} & \textbf{Incumbent} & \textbf{Reset} & \textbf{Best Expert} & \textbf{Average} \\
\midrule\multirow{3}{*}{CAPO-Avg}
  & Hopper   & $2836 \pm 521$ & $\mathbf{2991 \pm 260}$ & $2513 \pm 268$ & $2868 \pm 415$ \\
  & Walker2d & $3823 \pm 338$ & $2689 \pm 934$ & $\mathbf{4066 \pm 140}$ & $3836 \pm 238$ \\
  & Ant      & $1647 \pm 337$ & $1517 \pm 234$ & $\mathbf{2580 \pm 266}$ & $2534 \pm 240$ \\
\midrule
\multirow{3}{*}{CAPO}
  & Hopper   & $2041 \pm 530$ & $\mathbf{2697 \pm 578}$ & $2640 \pm 831$ & $1757 \pm 182$ \\
  & Walker2d & $\mathbf{3312 \pm 1052}$ & $3153 \pm 738$ & $2924 \pm 1072$ & $3197 \pm 592$ \\
  & Ant      & $\mathbf{2496 \pm 407}$ & $2016 \pm 257$ & $2328 \pm 90$ & $2229 \pm 34$ \\
\bottomrule
\end{tabular}
\end{table}

\textbf{No carry mode dominates consistently} (Table~\ref{tab:opt-carry}). Each generation, $K$ experts train independently from a shared Adam state, building up their own momentum over $E$ epochs. In \emph{incumbent} mode, expert optimizer states are discarded after aggregation; the pre-generation state is kept with only the step counter advanced so the learning rate schedule decays correctly. Since the momentum and second-moment buffers are never updated between generations, incumbent carry is effectively equivalent to \emph{reset} (both start each generation with zero momentum). The two modes differ only in that reset re-initializes the optimizer object, and indeed they perform comparably across all six settings. \emph{Best-expert} carry inherits one expert's Adam momentum, creating a potential mismatch when the next generation trains from the \emph{consensus} parameters; surprisingly, it performs well for CAPO-Avg (best on Walker2d and Ant), suggesting that the winning expert's curvature estimate transfers usefully under parameter averaging. \emph{Average} carry, which averages second moments across divergent gradient histories, is competitive for CAPO-Avg but consistently weak for CAPO (LogOP). Overall, the choice of carry mode is secondary to the aggregation method itself; we retain incumbent as the default for simplicity.

\subsection{Width Scaling: Varying $K$}
\label{app:k-sweep}

The epoch sweep (Section~\ref{sec:depth-dilemma}) shows that deeper optimization hits diminishing returns. The complementary question is: does wider optimization (more experts $K$ at fixed epoch count $E$) continue to improve? We run CAPO-Avg with $K \in \{2, 4, 8, 16\}$ at $E\!=\!10$ on Gymnasium, so that total gradient cost scales linearly with $K$.

\begin{table}[h]
\centering
\caption{Width scaling: CAPO and CAPO-Avg with varying $K$ at fixed $E\!=\!10$. Gymnasium, 1M steps, 3 seeds, mean $\pm$ SE.}
\label{tab:k-sweep}
\scriptsize
\setlength{\tabcolsep}{3pt}
\begin{tabular}{@{}lcccc@{}}
\toprule
 & $K\!=\!2$ & $K\!=\!4$ & $K\!=\!8$ & $K\!=\!16$ \\
\midrule
\multicolumn{5}{@{}l@{}}{\emph{HalfCheetah}} \\
CAPO     & $2491 \pm 824$ & $2382 \pm 765$ & $3022 \pm 668$ & $\mathbf{3556 \pm 965}$ \\
CAPO-Avg & $2021 \pm 319$ & $2510 \pm 728$ & $\mathbf{3116 \pm 562}$ & $1746 \pm 52$ \\
\midrule
\multicolumn{5}{@{}l@{}}{\emph{Walker2d}} \\
CAPO     & $3549 \pm 530$ & $2587 \pm 176$ & $\mathbf{4056 \pm 93}$ & $3585 \pm 558$ \\
CAPO-Avg & $3361 \pm 164$ & $3609 \pm 239$ & $\mathbf{4279 \pm 344}$ & $3523 \pm 34$ \\
\midrule
\multicolumn{5}{@{}l@{}}{\emph{Ant}} \\
CAPO     & $2346 \pm 310$ & $1939 \pm 96$ & $2356 \pm 189$ & $\mathbf{2801 \pm 227}$ \\
CAPO-Avg & $1675 \pm 341$ & $1392 \pm 71$ & $\mathbf{1840 \pm 252}$ & $1564 \pm 64$ \\
\bottomrule
\end{tabular}
\end{table}

Both methods are stable across $K$. On Walker2d, CAPO-Avg peaks at $K\!=\!8$ ($4279$) and remains above the PPO baseline at all widths; CAPO peaks at $K\!=\!8$ as well ($4056$). On HalfCheetah, CAPO-Avg peaks at $K\!=\!8$ ($3116$) but drops at $K\!=\!16$ ($1746$), while CAPO improves monotonically up to $K\!=\!16$ ($3556$). On Ant, CAPO scales well with width ($2801$ at $K\!=\!16$), while CAPO-Avg peaks at $K\!=\!8$ ($1840$). In all cases, width scaling does not degrade performance, in contrast to the catastrophic effect of increasing $E$ past $10$.

\subsubsection{Clip Threshold $\epsilon$}
\label{app:clip}

Table~\ref{tab:clip-eps-ablation} varies the PPO clip threshold.

\begin{table}[h]
\centering
\caption{Clip threshold ablation ($\epsilon \in \{0.1, 0.2, 0.3, 0.5\}$). Gymnasium, 1M steps, 3 seeds, mean $\pm$ SE.}
\label{tab:clip-eps-ablation}
\small
\begin{tabular}{@{}lcccc@{}}
\toprule
 & $\epsilon\!=\!0.1$ & $\epsilon\!=\!0.2$ & $\epsilon\!=\!0.3$ & $\epsilon\!=\!0.5$ \\
\midrule
\multicolumn{5}{@{}l@{}}{\emph{HalfCheetah}} \\
PPO & $\mathbf{1726 \pm 126}$ & $1568 \pm 196$ & $1165 \pm 38$ & $1068 \pm 91$ \\
CAPO-Avg & $1560 \pm 17$ & $\mathbf{3176 \pm 959}$ & $1631 \pm 45$ & $1809 \pm 213$ \\
CAPO & $1760 \pm 151$ & $\mathbf{2374 \pm 480}$ & $2035 \pm 415$ & $2067 \pm 767$ \\
\midrule
\multicolumn{5}{@{}l@{}}{\emph{Walker2d}} \\
PPO & $\mathbf{1544 \pm 479}$ & $932 \pm 223$ & $894 \pm 90$ & $462 \pm 34$ \\
CAPO-Avg & $2539 \pm 727$ & $2626 \pm 159$ & $\mathbf{3011 \pm 819}$ & $2890 \pm 423$ \\
CAPO & $865 \pm 79$ & $3850 \pm 611$ & $\mathbf{4001 \pm 404}$ & $2759 \pm 838$ \\
\midrule
\multicolumn{5}{@{}l@{}}{\emph{Ant}} \\
PPO & $\mathbf{503 \pm 78}$ & $316 \pm 47$ & $201 \pm 72$ & $68 \pm 16$ \\
CAPO-Avg & $990 \pm 40$ & $1588 \pm 187$ & $\mathbf{1930 \pm 426}$ & $1163 \pm 396$ \\
CAPO & $849 \pm 163$ & $2212 \pm 9$ & $2323 \pm 45$ & $\mathbf{2593 \pm 398}$ \\
\bottomrule
\end{tabular}
\end{table}

PPO degrades monotonically as $\epsilon$ increases on all three environments: from $1726$ to $1068$ on HalfCheetah, $1544$ to $462$ on Walker2d, and $503$ to $68$ on Ant.  Larger $\epsilon$ permits larger policy ratio deviations per epoch, compounding overshoot across multiple epochs.  CAPO-Avg and CAPO (LogOP) are substantially more robust: CAPO-Avg stays above $1560$ on HalfCheetah and above $2539$ on Walker2d across the full range.  On Ant, CAPO (LogOP) actually improves with $\epsilon$ ($849 \to 2593$), suggesting that the consensus mechanism can exploit larger per-expert updates when precision-weighting neutralizes the noise they introduce.

\subsubsection{KL Early Stopping}
\label{app:kl-early-stop}

Table~\ref{tab:target-kl-ablation} tests KL early stopping, which caps the number of epochs per iteration when a KL threshold is exceeded.

\begin{table}[h]
\centering
\caption{KL early stopping ablation ($\mathrm{target\_kl} \in \{0, 0.02, 0.04, 0.08\}$). Gymnasium, 1M steps, 3 seeds, mean $\pm$ SE.}
\label{tab:target-kl-ablation}
\small
\begin{tabular}{@{}lcccc@{}}
\toprule
 & $\mathrm{tkl}\!=\!0$ & $\mathrm{tkl}\!=\!0.02$ & $\mathrm{tkl}\!=\!0.04$ & $\mathrm{tkl}\!=\!0.08$ \\
\midrule
\multicolumn{5}{@{}l@{}}{\emph{HalfCheetah}} \\
PPO & $1475 \pm 114$ & $\mathbf{2922 \pm 756}$ & $2166 \pm 346$ & $1355 \pm 56$ \\
CAPO-Avg & $2522 \pm 992$ & $1579 \pm 30$ & $\mathbf{2532 \pm 814}$ & $1943 \pm 399$ \\
CAPO & $2274 \pm 676$ & $\mathbf{2361 \pm 821}$ & $1967 \pm 282$ & $2307 \pm 664$ \\
\midrule
\multicolumn{5}{@{}l@{}}{\emph{Walker2d}} \\
PPO & $664 \pm 94$ & $\mathbf{3281 \pm 339}$ & $1204 \pm 188$ & $1454 \pm 254$ \\
CAPO-Avg & $3160 \pm 1100$ & $\mathbf{4408 \pm 242}$ & $3955 \pm 189$ & $2490 \pm 776$ \\
CAPO & $3444 \pm 714$ & $3371 \pm 675$ & $3345 \pm 505$ & $\mathbf{3680 \pm 578}$ \\
\midrule
\multicolumn{5}{@{}l@{}}{\emph{Ant}} \\
PPO & $295 \pm 78$ & $\mathbf{2142 \pm 28}$ & $823 \pm 89$ & $406 \pm 34$ \\
CAPO-Avg & $1552 \pm 258$ & $2367 \pm 123$ & $2224 \pm 346$ & $1701 \pm 187$ \\
CAPO & $\mathbf{1851 \pm 203}$ & $\mathbf{2676 \pm 344}$ & $\mathbf{2643 \pm 104}$ & $\mathbf{2304 \pm 197}$ \\
\bottomrule
\end{tabular}
\end{table}

PPO on Ant is extremely sensitive to the KL threshold: it achieves $2142$ only at $\mathrm{tkl}\!=\!0.02$ and collapses to $295$--$823$ at all other settings, a ${\sim}7\!\times$ performance range.  CAPO-Avg and CAPO (LogOP) are far more robust, staying above $1550$ and $1851$ respectively across all settings.  On HalfCheetah and Walker2d, PPO likewise performs best at $\mathrm{tkl}\!=\!0.02$ (the only setting that prevents catastrophic overshoot without being overly conservative), while both CAPO variants maintain competitive performance across the full range.  Combined with the clip threshold results, this confirms that consensus aggregation stabilizes PPO against hyperparameter sensitivity.

\subsubsection{Warmup Budget}
\label{sec:warmup-ablation}

Table~\ref{tab:warmup-ablation} varies the fraction of the total step budget allocated to warmup PPO before expert averaging begins.

\begin{table}[h]
\centering
\caption{Warmup budget ablation. Fraction of total budget allocated to warmup PPO before expert averaging begins. Gymnasium, 1M steps, 5 seeds, mean $\pm$ SE.}
\label{tab:warmup-ablation}
\small
\begin{tabular}{@{}lcccc@{}}
\toprule
 & 0\% & 5\% & 10\% & 20\% \\
\midrule
\multicolumn{5}{@{}l@{}}{\emph{HalfCheetah}} \\
CAPO-Avg & $1710 \pm 36$ & $2443 \pm 417$ & $2469 \pm 537$ & $\mathbf{2706 \pm 473}$ \\
CAPO & $2203 \pm 496$ & $\mathbf{2512 \pm 535}$ & $2190 \pm 482$ & $2163 \pm 517$ \\
\midrule
\multicolumn{5}{@{}l@{}}{\emph{Walker2d}} \\
CAPO-Avg & $3695 \pm 521$ & $3193 \pm 550$ & $\mathbf{4089 \pm 226}$ & $3539 \pm 493$ \\
CAPO & $3455 \pm 658$ & $3394 \pm 484$ & $3430 \pm 278$ & $\mathbf{4263 \pm 405}$ \\
\midrule
\multicolumn{5}{@{}l@{}}{\emph{Ant}} \\
CAPO-Avg & $1623 \pm 77$ & $1417 \pm 180$ & $1658 \pm 161$ & $\mathbf{1696 \pm 369}$ \\
CAPO & $2231 \pm 65$ & $2194 \pm 46$ & $2267 \pm 196$ & $\mathbf{2384 \pm 147}$ \\
\bottomrule
\end{tabular}
\end{table}

The optimal warmup fraction is task-dependent. For CAPO-Avg, $10\%$ is best on Walker2d ($4089$) while $20\%$ is best on HalfCheetah ($2706$) and Ant ($1696$). CAPO (LogOP) consistently outperforms CAPO-Avg across warmup fractions and environments, peaking at $5\%$ on HalfCheetah ($2512$), $20\%$ on Walker2d ($4263$) and Ant ($2384$). The default $10\%$ used in the main results (Table~\ref{tab:main-results}) is a reasonable middle ground that avoids the worst-case on any environment for both methods.

\end{document}